\documentclass{article}

\usepackage{microtype}
\usepackage{graphicx}
\usepackage{subfigure}
\usepackage{booktabs} 

\usepackage{hyperref}

\usepackage[accepted]{icml2021}
\usepackage{graphicx}
\usepackage{xcolor}
\usepackage{amssymb}
\usepackage{upgreek}
\usepackage{amsmath}
\usepackage{bm}
\usepackage{amsfonts}
\usepackage{enumerate}
\usepackage{fancyhdr}
\usepackage{algorithm} 
\usepackage{algorithmic}
\usepackage{microtype}      
\newtheorem{lemma}{Lemma}  
\newtheorem{theorem}{Theorem}

\newenvironment{proof}{{\noindent\it Proof.}\quad}{\hfill $\square$\par}

\PassOptionsToPackage{square,comma,numbers,sort&compress}{natbib}
\usepackage[utf8]{inputenc} 
\usepackage[T1]{fontenc}    
\usepackage{hyperref}       
\usepackage{url}            
\usepackage{booktabs}       
\usepackage{amsfonts}       
\usepackage{nicefrac}       
\usepackage{microtype}      
\usepackage{xcolor}

\DeclareMathOperator*{\argmax}{arg\,max}

\icmltitlerunning{Model-based Reinforcement Learning for Continuous Control with Posterior Sampling}

\begin{document}

\twocolumn[
\icmltitle{Model-based Reinforcement Learning for Continuous Control with Posterior Sampling}



\icmlsetsymbol{equal}{*}

\begin{icmlauthorlist}
\icmlauthor{Ying Fan}{to}
\icmlauthor{Yifei Ming}{to}

\end{icmlauthorlist}

\icmlaffiliation{to}{University of Wisconsin-Madison}
\icmlcorrespondingauthor{Ying Fan}{yfan87@wisc.edu}
\icmlcorrespondingauthor{Yifei Ming}{ming5@wisc.edu}

\icmlkeywords{Machine Learning, ICML}

\vskip 0.3in
]



\printAffiliationsAndNotice{} 
\begin{abstract}
  Balancing exploration and exploitation is crucial in reinforcement learning (RL). In this paper, we study model-based posterior sampling for reinforcement learning (PSRL) in continuous state-action spaces theoretically and empirically. First, we show the first regret bound of PSRL in continuous spaces which is polynomial in the episode length to the best of our knowledge. With the assumption that reward and transition functions can be modeled by Bayesian linear regression, we develop a regret bound of $\tilde{O}(H^{3/2}d\sqrt{T})$, where $H$ is the episode length, $d$ is the dimension of the state-action space, and $T$ indicates the total time steps. This result matches the best-known regret bound of non-PSRL methods in linear MDPs. Our bound can be extended to nonlinear cases as well with feature embedding: using linear kernels on the feature representation $\phi$, the regret bound becomes $\tilde{O}(H^{3/2}d_{\phi}\sqrt{T})$, where $d_\phi$ is the dimension of the representation space. Moreover, we present MPC-PSRL, a model-based posterior sampling algorithm with model predictive control for action selection. To capture the uncertainty in models, we use Bayesian linear regression on the penultimate layer (the feature representation layer $\phi$) of neural networks. Empirical results show that our algorithm achieves the state-of-the-art sample efficiency in benchmark continuous control tasks compared to prior model-based algorithms, and matches the asymptotic performance of model-free algorithms. 
\end{abstract}
\section{Introduction}
\label{intro}
In reinforcement learning (RL), an agent interacts with an unknown environment which is typically modeled as a Markov Decision Process (MDP). Efficient exploration in RL has been one of the main challenges: the agent is expected to balance between {\it exploring} unseen state-action pairs to gain more knowledge about the environment, and {\it exploiting} existing knowledge to optimize rewards in the presence of known data. Specifically, when the state-action spaces are \emph{continuous}, function approximation is necessary to approximate the value function (in model-free settings) or the reward and transition functions (in model-based settings), which raises extra challenges for both computational and statistical efficiency compared to finite tabular cases\footnote{Tabular RL has been extensively studied with a regret bound of $\tilde{O}(H\sqrt{SAT})$, where $S$ and $A$ denote the number of states and actions respectively. However, in continuous state-action spaces $S$ and $A$ can be infinite, hence the above results do not apply to continuous spaces.}.


\paragraph{Frequentist Regrets with Upper Confidence Bound} Most existing works, which focus on using function approximations to achieve efficient exploration with performance guarantees, use algorithms based on Upper Confidence Bound (UCB) to develop frequentist regret bounds. In the model-based settings, the state-of-the-art frequentist bound is given by UC-MatrixRL \cite{yang2019reinforcement}, which achieves a regret bound of $\tilde{O}(H^2d\sqrt{T})$, where $H$ is the episode length, $d$ is the dimension of the state-action space, and $T$ indicates the total time steps.
In model-free settings, \citet{jin2020provably} proposed LSVI-UCB and developed a bound of $\tilde{O}(H^{3/2}d^{3/2}\sqrt{T})$. This bound is further improved to $\tilde{O}(H^{3/2}d\sqrt{T})$ by \citet{zanette2020learning} (model-free) and \citet{pmlr-v119-ayoub20a} (model-based), which achieves the best-known frequentist bound among model-free and model-based algorithms. All above bounds are achieved in linear MDPs where both reward and transition functions are modeled as linear functions. Those results can be extended to non-linear cases using kernel functions, replacing $d$ with $d_{\phi}$ where $d_{\phi}$ is the dimension of the feature space.
However, there are two main drawbacks of UCB-based methods. First, UCB requires optimizing over a confidence set, which is likely to be computationally prohibitive. Second, the lack of statistical efficiency can emerge from the sub-optimal construction of the confidence set \cite{osband17a}. Accordingly, this line of works mostly focuses on theoretical analysis rather than empirical applications. 

\paragraph{Bayesian Regrets with Posterior Sampling} Another line of works in Bayesian reinforcement learning treats MDP as a random variable with a prior distribution. 
This prior distribution of the MDP provides an initial uncertainty estimate of the environment, which generally contains distributions of transition dynamics and reward functions. The epistemic uncertainty (subjective uncertainty due to limited data) in reinforcement learning can be captured by posterior distributions given the data collected by the agent. Bayesian regrets naturally provide performance guarantees in this setting. 
Posterior sampling for reinforcement learning (PSRL), motivated by Thompson sampling in bandit problems  \citep{thompson1933likelihood}, serves as a provably efficient algorithm under Bayesian settings. In PSRL, the agent follows an optimal policy for a single MDP sampled from the posterior distribution for interaction in each episode, instead of optimizing over a confidence set, so PSRL is more computationally tractable than UCB-based methods.

Although PSRL with function approximation in continuous MDPs has been studied, to the best of our knowledge, there is no existing work that develops a regret bound which is \emph{clearly dependent in episode length $H$ with a polynomial order} while simultaneously sub-linear in $T$ as bounds developed via UCB. Existing results either provide no clear dependency on $H$ or suffer from an exponential order of $H$.

\paragraph{Limitations on the Order of H in PSRL with Function Approximation}

In model-based RL, \citet{osband2014model} derive a regret bound of  $\tilde{O}(\sigma_R\sqrt{d_K(R)d_E(R)T}+\mathbb{E}[L^*]\sigma_p\sqrt{d_K(P)d_E(P)})$ in their Corollary 1, where $L^*$ is a global Lipschitz constant for the future value function (see Section \ref{lips}), $d_K$ and $d_E$ are Kolmogorov and eluder dimensions, and $R$ and $P$ refers to function classes of rewards and transitions. However, $L^*$ is actually dependent on $H$ (see our remark in Section \ref{lips}).
Such dependency is not explored in their paper, and they didn't provide a clear dependency on $H$ in their Corollary 1. Moreover, they give a very loose bound (exponential in $H$) in their Corollary 2 for LQR, which we will discuss in detail in Section \ref{lqr}. 
\citet{chowdhury2019online} considers the regret bound for kernelized MDP which is sub-linear in $T$. However, they only mention that $L^*$ basically measure the connectedness of the MDP without discussing the dependency of $H$ in $L^*$ in continuous state-action spaces, and they followed \citet{osband2014model} in their Corollary 2 for LQR with an exponential order of $H$.
In model-free settings, \citet{azizzadenesheli2018efficient} develops a regret bound of $\tilde{O}(d_{\phi}\sqrt{T})$ using a linear function approximator in the Q-network, where $d_\phi$ is the dimension of the feature representation vector of the state-action space, but their bound is exponential in $H$ as mentioned in their paper.


Motivated by the drawbacks of previously discussed UCB-based and PSRL works, we are interested in the following question: in continuous MDPs, \emph{can PSRL achieve provably efficient exploration with polynomial orders of $d$ and $H$ in regret bounds, while still enjoy computational tractability of solving a single known MDP?}


\paragraph{Our Results}

In this paper, we study model-based PSRL in continuous state-action spaces. We assume that rewards and transitions can be modeled by Bayesian Linear regression \cite{rasmussen2003gaussian}, and extend the assumption to non-linear settings using feature representation. 

The key differences of our analysis compared to previous work in PSRL with function approximation are as follows: \textbf{First}, we show the order of $H$ can be polynomial in PSRL with continuous state-action spaces: 
in Section \ref{lips}, we use the property derived from any noise with a symmetric probability distribution, which includes many common noise assumptions, to derive a closed-form solution of the Lipschitz constant $L^*$ mentioned in \cite{osband2014model}. As a result, in Section \ref{lqr} we can develop a regret bound with polynomial dependency on $H$.
\textbf{Second}, our analysis requires less assumptions (especially compared to \citet{chowdhury2019online}): we omit their Lipschitz assumption (discussed in Section \ref{lips}) and regularity assumption (discussed in Section \ref{var}).
\textbf{Third}, our bound enjoys lower dimensionality (especially compared to \citet{osband2014model}) as discussed in Section  \ref{feature}.

To the best of our knowledge, we are the first to show that the regret bound for PSRL in continuous state-action spaces can be polynomial in the episode length $H$ and simultaneously sub-linear in $T$: For the linear case, we develop a Bayesian regret bound of $\tilde{O}(H^{3/2}d\sqrt{T})$. Using feature embedding, we derive a bound of $\tilde{O}(H^{3/2}d_{\phi}\sqrt{T})$. 
Our regret bound match the order of \emph{best-known} regret bound of UCB-based methods \citep{zanette2020learning,pmlr-v119-ayoub20a}, which is also $\tilde{O}(H^{3/2}d\sqrt{T})$.\footnote{We can compare them together in the Bayesian framework as discussed in \citet{osband17a}: A frequentist regret bound for a confidence set of MDPs implies an identical bound on the Bayesian regret for any prior distribution of MDPs with support on the same confidence set.} As far as is known, no previous works of UCRL have presented empirical results in continuous control. In contrast, PSRL only requires optimizing a \textbf{single} MDP, and thus enjoys more computationally tractability compared to UCB-based methods. 

Moreover, we implement PSRL with function approximation as a computationally tractable method with performance guarantee: We use Bayesian linear regression (BLR) \cite{rasmussen2003gaussian}  on the penultimate layer (for feature representation) of neural networks when fitting transition and reward models. We use model predictive control (MPC) \citep{camacho2013model}  as an approximate optimization solution of the sampled MDP, to optimize the policy under the sampled models in each episode as described in Section \ref{alg}. Experiments show that our algorithm achieves more efficient exploration compared with previous model-based algorithms in control benchmark tasks (see Section \ref{exp}).

\section{Preliminaries}

\subsection{Problem Formulation}
\label{formulation}
We model an episodic finite-horizon Markov Decision Process (MDP) $M$ as $\{\mathcal{S},\mathcal{A},R^{M},P^{M},H, \sigma_r, \sigma_f, R_{\text{max}},\rho\}$, where $\mathcal{S}\subset \mathbb{R}^{d_s}$ and $\mathcal{A}\subset \mathbb{R}^{d_a}$ denote state and action spaces, respectively.

Each episode with length $H$ has an initial state distribution $\rho$.  At time step $i \in [1,H]$ within an episode, the agent observes $s_i \in \mathcal{S}$, selects $a_i \in \mathcal{A}$, receives a noised reward  $r_i \sim R^M(s_{i},a_i)$ and transitions to a noised new state $s_{i+1} \sim P^M(\cdot|s_{i},a_i)$. More specifically, $r(s_{i},a_i) = \bar{r}^M(s_{i},a_i)+\epsilon_r$  
and $s_{i+1}=f^M(s_{i},a_i)+\epsilon_f$, where $\epsilon_r \sim \mathcal{N}(0,\sigma_r^2)$, $\epsilon_f \sim \mathcal{N}(0,\sigma_f^2I_{d_s})$. Variances $\sigma_r^2$ and $\sigma_f^2$ are fixed to control the noise level. Without loss of generality, we assume the expected reward an agent receives at a single step is bounded $|\bar{r}^M(s,a)| \leq R_{\text{max}}$, $\forall s\in \mathcal{S},a\in \mathcal{A}$. Let $\mu\colon\mathcal{S}\to \mathcal{A} $ be a deterministic policy. Define the value function for state $s$ at time step $i$ with policy $\mu$ as $V_{\mu,i}^M(s)=\mathbb{E}[\Sigma_{j=i}^{H}[\bar{r}^M(s_{j},a_{j})|s_i = s]$, where $s_{j+1}\sim P^M(\cdot|s_j,a_j)$ and $a_j=\mu(s_j)$. With the bounded expected reward, we have that  $|V(s)|\leq HR_{\text{max}}$, $\forall s$.

We use $M^*$  to indicate the real unknown MDP which includes $R^*$ and $P^*$, and $M^*$ itself is treated as a random variable. Thus, we can treat the real noiseless reward function $\bar{r}^*$ and transition function $f^*$ as random processes as well. In the posterior sampling algorithm $\pi^{PS}$,
$M^k$ is a random sample from the posterior distribution of the real unknown MDP $M^*$ in the $k$th episode, which includes the posterior samples of $R^k$ and $P^k$ , given history prior to the $k$th episode: $\mathcal{H}_{k} :=\{s_{1,1},a_{1,1},r_{1,1},\cdots,s_{k-1,H},a_{k-1,H},r_{k-1,H} \}$, where $s_{k,i}, a_{k, i}$ and $r_{k, i}$ indicate the state, action, and reward at time step $i$ in episode $k$. We define the the optimal policy under $M$ as $\mu^M(s_{i})\in \argmax_{\mu_2} V_{\mu_2,i}^M(s_{i})$. In particular, $\mu^*$ indicates the optimal policy under $M^*$ and $\mu^{k}$ represents the optimal policy under $M^{k}$. Define future value function:  $U_i^M(P)=\mathbb{E}_{s'\sim P(s')}[ V_{\mu^M,i+1}^M(s')]$. Let $\Delta_k$ denote the regret over the $k$th episode:
\begin{equation}
\Delta_k = \int \rho(s_1) (V_{\mu^*,1}^{M^*}(s_1)-V_{\mu^k, 1}^{M^*}(s_1)) ds_1
\end{equation}

Then we can express the regret of $\pi^{ps}$ up to time step T as:
\begin{equation}
Regret(T,\pi^{ps},M^*):=\Sigma_{k=1}^{\lceil\frac{T}{H}\rceil} \Delta_k,
\end{equation}

Let $BayesRegret(T,\pi^{ps}, \phi)$ denote the Beyesian regret of $\pi^{ps}$ as defined in \citet{osband17a}, where $\phi$ is the prior distribution of $M^*$:
\begin{equation}
    BayesRegret(T,\pi^{ps}, \phi) = \mathbb{E}[Regret(T,\pi^{ps},M^*\sim \phi)].
\end{equation}

\subsection{Gaussian Process Assumption}
\label{gp}

Generally, we consider modeling an unknown target function $g: \mathbb{R}^d\rightarrow\mathbb{R}$.
We are given a set of noisy samples $y =[y_1....,y_T]^T$ at points $  {X} = [ {x}_1,...,  {x}_T]^T$, $  {X} \subset D$, where $D$ is compact and convex, $y_i = g(x_i)+\epsilon_i$ with $\epsilon_i \sim N(0,\sigma^2)$ i.i.d. Gaussian noise $\forall i \in \{1,\cdots,T\}$.

We model $g$ as a sample from a Gaussian Process $GP(\mu({x}),\mathcal{K}({x},{x}'))$, specified by the mean function $\mu({x}) = \mathbb{E}[g({x})]$ and the covariance (kernel) function $\mathcal{K}({x},{x}')=\mathbb{E}[(g({x})-\mu({x})(g({x}')-\mu({x}')]$.

Let the prior distribution without any data as $GP({0},\mathcal{K}({x},{x}'))$.  Then the posterior distribution over $g$ given $X$ and $y$ is also a GP with mean $\mu_{T}(x)$, covariance $\mathcal{K}_{T}(x, x')$, and variance $\sigma_{T}^2(x)$ \cite{rasmussen2003gaussian}:
\begin{equation*}
\begin{split}
&\mu_{T}(x) = \mathcal{K}(x,X)(\mathcal{K}(X,X)+\sigma^2 {I})^{-1} {y},\\
&\mathcal{K}_{T}(  {x},  {x}') = \mathcal{K}(  {x},{x}')\\
&-\mathcal{K}(X,x)^T(\mathcal{K}(X,X)+\sigma^2I)^{-1}\mathcal{K}(X,x),\\
&\sigma^2_{T}({x}) = \mathcal{K}_{T}(  {x},  {x}),\\
\end{split}
\end{equation*}
where $\mathcal{K}(X,x) = [\mathcal{K}({x}_1,  {x}),...,\mathcal{K}(  {x}_T,  {x})]^T$,
$\mathcal{K}(X,X)=[\mathcal{K}(  {x_i},  {x_j})]_{1\leq i \leq T, 1\leq j \leq T}$.

We model our reward function $\bar{r}^M$ as a Gaussian Process with noise $\sigma^2_r$.
 
For transition models, we treat each dimension independently: each $f_i(s,a), i=1,..,d_S$ is modeled independently as above, and with the same noise level $\sigma^2_f$ in each dimension. Thus it corresponds to our formulation in the RL setting.
Since the posterior covariance matrix is only dependent on the input rather than the target value, the distribution of each $f_i(s, a)$ shares the same covariance matrice and only differs in the mean function.

\section{Bayesian Regret Analysis}
\label{analysis}
\subsection{Regret Decomposition}

In this section, we briefly describe the regret decomposition as presented in \citet{osband2014model} to facilitate further analysis.

The regret in episode $k$ can be rearranged as:
$\Delta_k= \int \rho(s_1)( V_{\mu^*,}^{M^*}(s_1)-V_{\mu^k,1}^{M^k}(s_1))+\tilde{\Delta}_k)ds_1,$  where $\tilde{\Delta}_k=V_{\mu^k,1}^{M^k}(s_1)-V_{\mu^k,1}^{M^*}(s_1)$.
In PSRL, $V_{\mu^*,}^{M^*}-V_{\mu^k,1}^{M^k}$ is zero in expectation, and thus we only need to bound $\tilde{\Delta}_k$ when deriving the Bayesian regret of PSRL.\footnote{It suffices to derive bounds for any initial state $s_1$ as the regret bound will still hold through the integration of the initial distribution $\rho(s_1)$.}
For clarity, the value function $V_{\mu^{k},1}^{M^k}$ is simplified to $V_{k,1}^{k}$ and $V_{\mu^{k},1}^{M^*}$ to $V_{k,1}^{*}$. Let $h_i$ indicate the state-action pair $s_{i},a_i$, and $h_i$ is the state-action pair that the agent encounters in the $k$th episode while using $\mu_k$ as policy in the real MDP $M^*$ \footnote{where $a_i = \mu_k(s_i), s_{i+1} \sim P^*(\cdot|s_i,a_i)$. }. 
Consider the regret from concentration via the Bellman operator (details of the derivation can be found in \cite{osband2014model}):
$$\mathbb{E}[\tilde{\Delta}_k|\mathcal{H}_{k}] =\mathbb{E}[\tilde{\Delta}_k(r)+\tilde{\Delta}_k(f)|\mathcal{H}_{k}],$$

where $\tilde{\Delta}_k(r) = \Sigma_{i=1}^H (\bar{r}^k(h_i)-\bar{r}^*(h_i)),$ and $\tilde{\Delta}_k(f) = \Sigma_{i=1}^H
(U_{i}^k(P^k(h_i))-U_{i}^k(P^*(h_i)))
.$

\subsection{Dependency on H in the Lipschitz Constant of Future Value Functions}
\label{lips}
First, we present a lemma on the property derived from noises with a symmetric probability distribution.

\begin{lemma}
\label{noise}
(Property derived from noises with symmetric probability distribution) Consider two zero-mean noises $\bm{\epsilon_1}, \bm{\epsilon_2} \in \mathbb{R}^d$, and the noise in each dimension of $\bm{\epsilon_1}, \bm{\epsilon_2}$ is i.i.d. drawn from the same symmetric probability distribution. Let $P_1, P_2$ be the probability distribution of the random variables $\bm{\mu_1+\epsilon_1}$ and $\bm{\mu_2+\epsilon_2}$ respectively, where $\bm{\mu_1,\mu_2} \in \mathbb{R}^d$.

Then we have 
$$||P_1-P_2|| \leq 
C
|| \bm{\mu}_1- \bm{\mu}_2||_2,$$
where $C$ is a constant that is only dependent on the variance of the noise.
\end{lemma}
The proof of Lemma \ref{noise} is in Appendix.  Using this Lemma, we can develop a closed-form upper bound of the Lipschitz constant of the future value function.
\begin{lemma} (The dependency of $H$ in the future value function) We have
\begin{equation}
\begin{split}
&U_{i}^k(P^k(h_i))-U_{i}^k(P^*(h_i))\\
&\leq CHR_{\text{max}}||f^k(h_i)-f^*(h_i)||_2,
\end{split}
\end{equation}
where $C$ is the constant mentioned in Lemma \ref{noise}.
\end{lemma}

\begin{proof}
For all $i$, we have
\begin{equation}
\begin{split}
&U_{i}^k(P^k(h_i))-U_{i}^k(P^*(h_i))\\
&\leq \max_{s}|V_{k,i+1}^k(s)| ||P^k(\cdot|h_i)-P^*(\cdot|h_i)||\\
&\leq HR_{\text{max}}||P^k(\cdot|h_i)-P^*(\cdot|h_i)||
\end{split}
\end{equation}

Recall that $P^k(s'|h_i)=\mathcal{N}(f^k(h_i),\sigma_f^2 {I})$ and $P^*(s'|h_i)=\mathcal{N}(f^*(h_i),\sigma_f^2 {I})$. By Lemma \ref{noise} we have
\begin{equation}
    ||P^k(\cdot|h_i)-P^*(\cdot|h_i)||\leq C||f^k(h_i)-f^*(h_i)||_2,
\end{equation} 
And the proof will be complete by combining (4) and (6).
\end{proof}

\paragraph{Remark} Lemma \ref{noise} is crucial for developing a bound which is polynomially dependent on $H$ in Section \ref{lqr}. Here $CHR_{\text{max}}$ serves as the Lipschitz constant in the Lipschitz assumption of the future value function in \citet{osband2014model} . In fact, the Lipschitz constant here is naturally dependent on $H$: when $i=1$, the future value function includes the cumulative rewards within $H$ steps. As the distance between two initial states propagates in $H$ steps (and result in differences in the rewards), the resulting difference in the future value function is naturally dependent on $H$. However, \citet{osband2014model} did not explore such dependency and directly assume the Lipschitz continuity of the future value function their Corollary 1 (presented in Section \ref{intro}). In contrast, we present the Lipschitz continuity of the future value function as a result of the property of noises and provide its dependency on $H$.

\subsection{Connecting Regret with Posterior Variances}
\label{var}
In this section, we show the upper bound of $\Sigma_{k=1}^{[\frac{T}{H}]}\tilde{\Delta}_k(f)$ conditioned on any given history $\mathcal{H}_k$ with high probability.
\begin{lemma} (Upper bound by the sum of posterior variances)
With probability at least $1-\delta$, 
\begin{equation}
\begin{split}
&\Sigma_{k=1}^{[\frac{T}{H}]}[\tilde{\Delta}_k(f)|\mathcal{H}_{k}] \\
&\leq \Sigma_{k=1}^{[\frac{T}{H}]}4CH^2R_{\text{max}}\sqrt{2d_s\sigma^2_k(h_{\text{kmax}})log\frac{4Td_s}{\delta}}.\\
\end{split}
\end{equation}

\end{lemma}
\begin{proof}
Given history $\mathcal{H}_{k}$, let $\bar{f}^k(h)$ indicate the posterior mean of $f^k(h)$ in episode $k$, and $\sigma^2_k(h)$ denotes the posterior variance of $f^k$ in each dimension. Note that $f^*$ and $f^k$ share the same variance in each dimension given history $\mathcal{H}_{k}$, as described in Section 3. Consider all dimensions of the state space, We have that for $N$ sub-Gaussian random variables: $X_1, . . . ,X_N$  with variance $\sigma^2$ (not required to be independent), and for any $t>0$, $\mathbb{P}(\max_{1\leq i\leq N}|X_i|>t)\leq 2N e^{-\frac{t^2}{2\sigma^2}}.$ \citep{rigollet2015high}. So with probability at least $1-\delta$, for any state-action pair $h$, $\max_{1\leq i \leq d_s}|f^k_i(h)-\bar{f}^k_i(h)|\leq \sqrt{2\sigma^2_k(h)\log\frac{2d_s}{\delta}}.$ Also, we can derive an upper bound for the norm of the state difference $||f^k(h)-\bar{f}^k(h)||_2 \leq \sqrt{d_s} \max_{1\leq i \leq d_s}|f^k_i(h)-\bar{f}^k_i(h)|$,
and so does $||f^*(h)-\bar{f}^k(h)||_2$ since $f^*$ and $f^k$ share the same posterior distribution. By the union bound, we have that with probability at least $1-2\delta$,  $||f^k(h)-f^*(h)||_2\leq 2\sqrt{2d_s\sigma^2_k(h)\log\frac{2d_s}{\delta}}$.

Then we look at the sum of the differences over horizon $H$, without requiring each variable in the sum to be independent: 
\begin{equation}
\begin{split}
&\mathbb{P}(\Sigma_{i=1}^H||f^k(h_i)-f^*(h_i)||_2> \Sigma_{i=1}^H 2\sqrt{2d_s\sigma^2_k(h_i)\log\frac{2d_s}{\delta}})\\
&\leq  \mathbb{P}(\bigcup\limits_{i=1}^{H}\{||f^k(h_i)-f^*(h_i)||_2> 2\sqrt{2d_s\sigma^2_k(h_i)\log\frac{2d_s}{\delta}}\})\\
&\leq \Sigma_{i=1}^H \mathbb{P}(||f^k(h_i)-f^*(h_i)||_2> 2\sqrt{2d_s\sigma^2_k(h_i)\log\frac{2d_s}{\delta}})
\end{split}
\end{equation}
Thus, 
we have that with probability $1-\delta$, 
\begin{equation}
\begin{split}
&\Sigma_{i=1}^H||f^k(h_i)-f^*(h_i)||_2 \\
&\leq \Sigma_{i=1}^H 2\sqrt{2d_s\sigma^2_k(h_i)\log\frac{4Hd_s}{\delta}}\\ 
&\leq 2H\sqrt{2d_s\sigma^2_k(h_{\text{kmax}})\log\frac{4Hd_s}{\delta}},
\end{split}
\end{equation}
where we define the index: ${\text{kmax}}:=\argmax_{i}\sigma_k(h_i), i=1,...,H$ in episode $k$. Here, since the posterior distribution is only updated every H steps, we have to use data points with the max variance in each episode to bound the result. 
Similarly, using the union bound for $[\frac{T}{H}]$ episodes, we have that with probability at least $1-\delta$, 
\begin{equation*}
\begin{split}
&\Sigma_{k=1}^{[\frac{T}{H}]}[\tilde{\Delta}_k(f)|\mathcal{H}_{k}]\\
&\leq \Sigma_{k=1}^{[\frac{T}{H}]} \Sigma_{i=1}^H2CHR_{\text{max}}||f^k(h_i)-f^*(h_i)||_2\\ 
&\leq \Sigma_{k=1}^{[\frac{T}{H}]}4CH^2R_{\text{max}}\sqrt{2d_s\sigma^2_k(h_{\text{kmax}})\log\frac{4Td_s}{\delta}}.
\end{split}
\end{equation*}

\end{proof}
\paragraph{Remark}  Here we bound $\tilde{\Delta}_k(f)$ using \emph{point-wise} concentration properties of rewards and transitions that applies to any state-action pair. Then we use the union bound on each state-action pair that the agent encounters in every episode. In contrast, \citet{chowdhury2019online} use the uniform concentration on the reward and transition functions, which requires an extra assumption (their regularity assumption of the RKHS norm) compared to our analysis.

\subsection{Regret with Linear Kernels}
\label{lqr}
\begin{theorem}
In the RL problem formulated in Section \ref{formulation}, under the assumption of Section \ref{gp} with linear kernels \footnote{GP with linear kernel correspond to Bayesian linear regression $f(x) = w^Tx$, where the prior distribution of the weight is $w \sim \mathcal{N}(0,\Sigma_p)$ \cite{rasmussen2003gaussian} .}, we have  $BayesRegret(T,\pi^{ps}, \phi)  =\tilde{O}(H^{3/2}d\sqrt{T})$, where $d$ is the dimension of the state-action space, $H$ is the episode length, and $T$ is the time elapsed.
\end{theorem}
\begin{proof}
In each episode $k$, let $\sigma_k^{'2}(h)$ denote the posterior variance given only a subset of data points $\{h_{\text{1max}},...,h_{\text{(k-1)max}}\}$, where each element has the max variance in the corresponding episode. By Eq.(6) in \citet{williams2000upper},  we know that the posterior variance reduces as the number of data points grows. Hence $\forall h, \sigma_k^2(h)\leq \sigma_k^{'2}(h)$.
By Theorem 5 in \citet{srinivas2012} which provides a bound on the information gain, and Lemma 2 in \citet{russo2014learning} that bounds the sum of variances by the information gain, we have that  $\Sigma_{k=1}^{[\frac{T}{H}]} \sigma_k^{'2}(h_{\text{kmax}}) = \mathcal{O} ((d_s+d_a)\log[\frac{T}{H}])$ for linear kernels with bounded variances (See Appendix for details).

Thus with probability $1-\delta$, and let $\delta = \frac{1}{T}$,
\begin{equation}
\begin{split}
&\Sigma_{k=1}^{[\frac{T}{H}]}[\tilde{\Delta}_k(f)|\mathcal{H}_{k}]\\ 
&\leq \Sigma_{k=1}^{[\frac{T}{H}]}4CH^2R_{\text{max}}\sqrt{2d_s\sigma^2_k(h_{\text{kmax}})\log\frac{4Td_s}{\delta}}\\
&\leq 8CH^2R_{\text{max}}\sqrt{\Sigma_{k=1}^{[\frac{T}{H}]}\sigma_k^{'2}(h_{\text{kmax}})}\sqrt{[\frac{T}{H}]}\sqrt{d_s \log(2Td_s)}\\
&\leq 8CH^{\frac{3}{2}}R_{\text{max}}\sqrt{T} \sqrt{d_s \log(2Td_s)} \sqrt{\mathcal{O} ((d_s+d_a)\log(T))}\\
&=\tilde{\mathcal{O}}((d_s+d_a)H^{\frac{3}{2}}\sqrt{T})
\end{split}  
\end{equation}
where $\tilde{\mathcal{O}}$ ignores logarithmic factors of $T$.

Therefore, 
\begin{equation}
\begin{split}
&\mathbb{E}[\Sigma_{k=1}^{[\frac{T}{H}]}\tilde{\Delta}_k(f)|\mathcal{H}_{k}]\\
&\leq (1-\frac{1}{T}) \tilde{\mathcal{O}}((d_s+s_a)H^{\frac{3}{2}}T)+\frac{1}{T} 2HR_{\text{max}} [\frac{T}{H}]\\
&= \tilde{\mathcal{O}}(H^{\frac{3}{2}}d\sqrt{T}),
\end{split}
\end{equation} where $2HR_{\text{max}}$ is the upper bound on the difference of value functions, and $d = d_s+d_a$. 
Following similar derivation,  $\mathbb{E}[\Sigma_{k=1}^{[\frac{T}{H}]}\tilde{\Delta}_k(r)|\mathcal{H}_{k}]= \tilde{\mathcal{O}}(\sqrt{dHT})$ (See Appendix for details). Finally, through the tower property we have $BayesRegret(T,\pi^{ps}, M^*)=\tilde{\mathcal{O}}(H^{\frac{3}{2}}d\sqrt{T})$.
\end{proof}

\paragraph{Remark} Here we compare our result with Corollary 2 in \cite{osband2014model}. Note that we maintain the same assumptions of transitions and rewards as \cite{osband2014model}. However, in their Corollary 2 which describes the regret for LQR, they directly use the Lipschitz constant of the \emph{underlying value function}, instead of the future value function. The Lipschitz constant of the underlying function in LQR is actually \emph{exponential} in $H$; as a result, even if the reward is linear, their bound would still be exponential in $H$ (See Appendix for details), while we present a regret bound polynomial in $H$. So their Corollary 2 is very loose and can be improved by our analysis.

\subsection{Nonlinear Extension via Feature Representation}
\label{feature}
We can slightly modify the previous proof to derive the bound in settings that use feature representations. 
Consider the mapping: $(s,a)\rightarrow s$ in the transition model. 
We transform the state-action pair $(s,a)$ to $\phi_f(s,a)\in \mathbb{R}^{d_{\phi}}$ as the input of the transition model, and transform the target $s'$ to $\psi_f(s')\in \mathbb{R}^{d_{\psi}}$ , then this transition model can be established with respect to this feature embedding. We further assume $d_{\psi} = O(d_{\phi})$ 
. Besides, we assume $d_{\phi'}=O(d_{\phi})$ in the feature representation $\phi_r(s,a)\in \mathbb{R}^{d_{\phi'}}$, then the reward model can also be established with respect to the feature embedding. In this way, we can handle non-linear rewards and transitions with only linear kernels in GP. Following similar steps in previous analysis will lead to a Bayesian regret of $\tilde{O}(H^{3/2}d_{\phi}\sqrt{T})$.

Empirically, the linearity of rewards and transitions can be preserved by updating the feature representation. By updating representations of all state-action pairs in the history and the covariance correspondingly, the theoretical extension to nonlinear cases still holds in practice.

\paragraph{Remark} 

The eluder dimension of neural networks in \citet{osband2014model} can blow up to infinity, and the information gain used in \citet{chowdhury2019online} yields exponential order of dimension $d$ if nonlinear kernels are used, such as SE and Matérn kernels.
But linear kernel can only model linear functions, thus the representation power is restricted if the polynomial order of $d$ is desired in their result. 
We first derive results for linear kernels, and increase the representation power by extracting the penultimate layer of neural networks, and thus we can derive a bound linear in the dimension of the penultimate layer, which is generally much less than the exponential order of the input dimension of neural networks.

\section{Algorithm Description}
\label{alg}
In this section, we elaborate our proposed algorithm, MPC-PSRL, as shown in Algorithm \ref{alg:MPC-PSRL}. 

\begin{algorithm}[tb]
  \caption{MPC-PSRL}
  \label{alg:MPC-PSRL}
\begin{algorithmic}
  \STATE Initialize data $\mathcal{D}$ with random actions for one episode
  \REPEAT 
  \STATE Sample a transition model and a cost model at the beginning of each episode
  \FOR{$i=1$ to $H$ steps}
  \STATE Obtain action using MPC with planning horizon $\tau$: $a_i \in \arg\max_{a_{i:i+\tau}} \sum_{t=i}^{i+\tau} \mathbb{E}[r(s_t, a_t)]$
  \STATE $\mathcal{D} = \mathcal{D} \cup \lbrace (s_i, a_i, r_i, s_{i+1}) \rbrace$
  \ENDFOR
  \STATE Train cost and dynamics representations $\phi_r$ and $\phi_f$ using data in $\mathcal{D}$
  \STATE Update $\phi_r(s,a)$, $\phi_f(s,a)$ for all $(s,a)$ collected
  \STATE Perform posterior update of $w_r$ and $w_f$ in cost and dynamics models using updated representations $\phi_r(s,a)$, $\phi_f(s,a)$ for all $(s,a)$ collected

  \UNTIL{convergence}
\end{algorithmic}
\end{algorithm}

\subsection{Predictive Model}
\label{model}
When modeling the rewards and transitions, we use features extracted from the penultimate layer of fitted neural networks, and perform Bayesian linear regression on the feature vectors to update posterior distributions.

\begin{figure*}[h!]
\makebox[\textwidth][c]
{\includegraphics[width=0.7\textwidth]{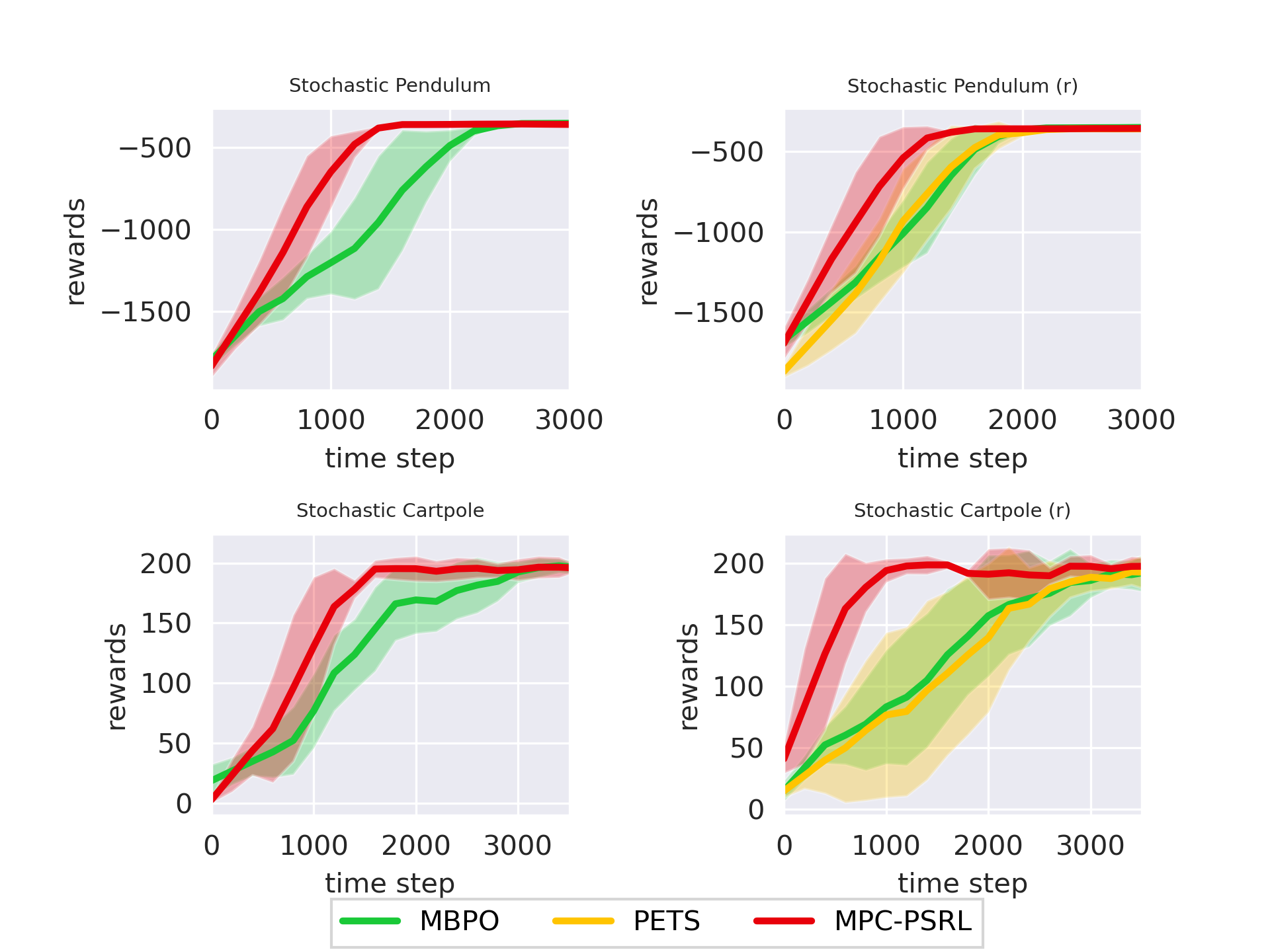}}

\caption{Training curves of MPC-PSRL (shown in red), and other model-based baseline algorithms in stochastic tasks. Solid curves are the mean of five trials, shaded areas correspond to the standard deviation among trials. (r) means with oracle rewards provided.}
\label{fig_new}
\end{figure*}

\textbf{Feature representation:} 
we first fit neural networks for transitions and rewards, using the same network architecture as \citet{chua2018deep}.  Let $x_i$ denote the state-action pair $(h_i)$ and $y_i$ denote the target value. Specifically, we use reward $r_i$ as $y_i$ to fit rewards, and we take the difference between two consecutive states $s_{i+1} - s_{i}$ as $y_i$ to fit transitions. 
The penultimate layer of fitted neural networks is extracted as the feature representation, denoted as $\phi_f$ and $\phi_r$ for transitions and rewards, respectively. 
Note that in the transition feature embedding, we only use one neural network to extract features of state-action pairs from the penultimate layer to serve as $\phi$, and leave the target states without further feature representation (the general setting is discussed in Section \ref{feature} where feature representations are used for both inputs and outputs), so the dimension of the target in the transition model $d_{\psi}$ equals to $d_s$. Thus we have a modified regret bound of $\tilde{O}(H^{3/2}\sqrt{dd_{\phi}T})$.  We do not find the necessity to further extract feature representations in the target space, as it might introduce additional computational overhead. Although higher dimensionality of the hidden layers might imply better representation, we find that only modifying the width of the penultimate layer to be the same order of $d=d_s+s_a$ suffices in our experiments for both reward and transition models. Note that how to optimize the dimension of the penultimate layer for more efficient feature representation deserves further exploration.

\textbf{Bayesian update and posterior sampling:}
here we describe the Bayesian update of transition and reward models using extracted features. Recall that Gaussian process with linear kernels is equivalent to Bayesian linear regression. By extracting the penultimate layer as feature representation $\phi$, the target value $y$ and the representation $\phi(x)$ could be seen as linearly related: $y = w\top\phi(x)+\epsilon$, where $\epsilon$ is a zero-mean Gaussian noise with variance $\sigma^2$ (which is $\sigma^2_f$ for the transition model and $\sigma^2_r$ for the reward model as defined in Section \ref{formulation}). We choose the prior distribution of weights $w$ as zero-mean Gaussian with covariance matrix $\Sigma_p$, then the posterior distribution of $w$ is also multivariate Gaussian (\citet{rasmussen2003gaussian}): $$p(w|\mathcal{D}) \sim \mathcal{N}\left( \sigma^{-2} A^{-1}\Phi  {Y}, A^{-1}\right) $$
where $A = \sigma^{-2}\Phi\Phi^\top + \Sigma_p^{-1}$, $\Phi\in\mathcal{R}^{d_{\phi}\times N}$ is the concatenation of feature representations $\lbrace\phi(x_i)\rbrace_{i=1}^{N}$, and $  {Y}\in\mathcal{R}^{N}$ is the concatenation of target values. At the beginning of each episode, we sample $w$ from the posterior distribution to build the model, collect new data during the whole episode, and update the posterior distribution of $w$ at the end of the episode using all the data collected. Here we present the complexity for posterior sampling: The matrix multiplication for covariance matrix $A$  is $O(d_{\phi}^2N)$; The inverse of $A$ is $O(d_{\phi}^3)$.

Besides the posterior distribution of $w$, the feature representation $\phi$ is also updated in each episode with new data collected. We adopt a similar dual-update procedure as \citet{riquelme2018deep}: after representations for rewards and transitions are updated, feature vectors of all state-action pairs collected are re-computed. Then we apply Bayesian update on these feature vectors. See the description of Algorithm \ref{alg:MPC-PSRL} for details.

\subsection{Planning}

During interaction with the environment, we use a MPC controller (\citet{camacho2013model}) for planning. At each time step $i$, the controller takes state $s_i$ and an action sequence $a_{i:i+\tau} = \lbrace a_i, a_{i+1}, \cdots, a_{i+\tau} \rbrace $ as the input, where $\tau$ is the planning horizon. We use transition and reward models to produce the first action $a_i$ of the sequence of optimized actions $\arg\max_{a_{i:i+\tau}} \sum_{t = i} ^{i + \tau} \mathbb{E}[r(s_t, a_t)]$, where the expected return of a series of actions can be approximated using the mean return of several particles propagated with noises of our sampled reward and transition models. To compute the optimal action sequence, we use CEM (\citet{botev2013cross}), which samples actions from a distribution closer to previous action samples with high rewards. 

\begin{figure*}[h!]
\makebox[\textwidth][c]
{\includegraphics[width=1.01\textwidth]{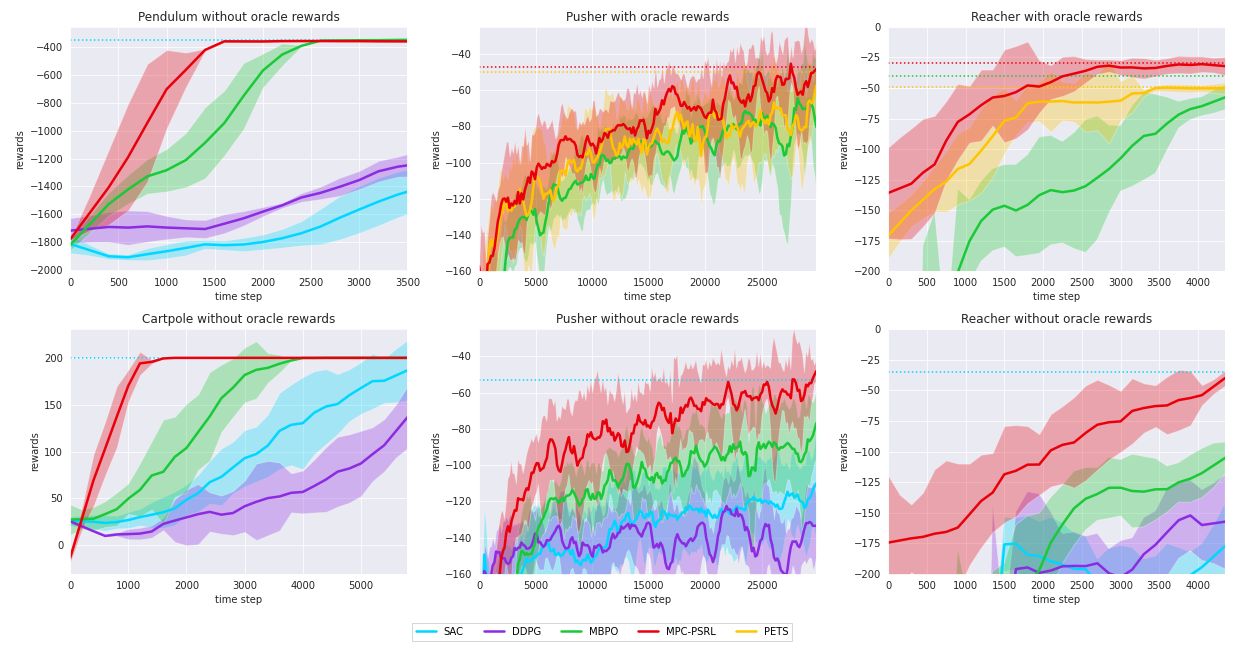}}

\caption{Training curves of MPC-PSRL (shown in red), and other baseline algorithms in different tasks. Solid curves are the mean of five trials, shaded areas correspond to the standard deviation among trials, and the dotted line shows the rewards at convergence.}
\label{fig}
\end{figure*}

\section{Experiments}
\label{exp}

\subsection{Baselines}
We compare our method with the following state-of-the-art model-based and model-free algorithms on benchmark control tasks. 

\textbf{Model-free:}
Soft Actor-Critic (SAC) from \citet{haarnoja2018soft} is an off-policy deep actor-critic algorithm that utilizes entropy maximization to guide exploration. Deep Deterministic Policy Gradient (DDPG) from \citet{barth2018distributed} is an off-policy algorithm that concurrently learns a Q-function and a policy, with a discount factor to guide exploration.

\textbf{Model-based:} Probabilistic Ensembles with Trajectory Sampling (PETS) from \citet{chua2018deep} models the dynamics via an ensemble of probabilistic neural networks to capture epistemic uncertainty for exploration, and uses MPC for action selection, with a requirement to have access to oracle rewards for planning. Model-Based Policy Optimization (MBPO) from \citet{janner2019trust} uses the same bootstrap ensemble techniques as PETS in modeling, but differs from PETS in policy optimization with a large amount of short model-generated rollouts, and can cope with environments with no oracle rewards provided.
%
We do not compare with \citet{gal2016improving}, which adopts a single Bayesian neural network (BNN) with moment matching, as it is outperformed by PETS that uses an ensemble of BNNs with trajectory sampling. 
And we don't compare with GP-based trajectory optimization methods with real rewards provided (\citet{deisenroth2011pilco}, \citet{kamthe2018data}), which are not only outperformed by PETS, but also computationally expensive and thus are limited to very small state-action spaces.

\subsection{Environments and Results}
We use environments with various complexity and dimensionality for evaluation:

Low-dimensional stochastic environments: continuous Cartpole ($d_s = 4$, $d_a = 1,H=200$, with a continuous action space compared to the classic Cartpole, which makes it harder to learn) and Pendulum Swing Up ($d_s = 3$, $d_a = 1,H=200$, a modified version of Pendulum where we limit the start state to make it harder for exploration). The transitions and rewards are originally deterministic in these environments, so we first modify the physics in Cartpole and Pendulum so that the transitions are stochastic with independent Gaussian noises ($\mathcal{N}(0,0.01)$). We also use noises in the same form for stochastic rewards. The learning curves of model-based algorithms are shown in Figure \ref{fig_new}, which shows our algorithm significantly outperforms model-based baselines in these stochastic environments.

Higher-dimensional environments: 7-DOF Reacher ($d_s=17,d_a=7,H=150$) and 7-DOF pusher ($d_s=20,d_a=7,H=150$) are two more challenging tasks as provided in \cite{chua2018deep}, where we conduct experiments both with and without true rewards, to compare with all baseline algorithms mentioned. 

The learning curves of all compared algorithms are shown in Figure \ref{fig}, and the hyperparameters and other experimental settings in our experiments are provided in Appendix. Here we also included results from deterministic Cartpole and Pendulum without oracle rewards. When their trajectories are deterministic, optimization with oracle rewards in these two environments becomes very easy and there is no significant difference in the performances for all model-based algorithms we compare, so we omit those learning curves in Figure \ref{fig}. However, when we add noise to the trajectory, these environments become harder to learn even when the rewards are provided, and we can observe the difference in the performance of different algorithms as in Figure \ref{fig_new}.

When the oracle rewards are provided in Pusher and Reacher, our method outperforms PETS and MBPO: it converges more quickly with similar performance at convergence in Pusher, while in Reacher, not only does it learn faster but also performs better at convergence. As we use the same planning method (MPC) as PETS, results indicate that our model better captures the uncertainty, which is beneficial to improving sample efficiency. When exploring in environments where both rewards and transition are unknown, our method significantly outperforms previous model-based and model-free methods which do no require oracle rewards. Meanwhile, it matches the performance of SAC at convergence. Convergence results are provided in Appendix.

\subsection{Discussion}
In our experiment, we have shown that our model-based algorithm outperforms \citet{chua2018deep} and \citet{janner2019trust} in given environments. 
Notice that \citet{chua2018deep}, \citet{janner2019trust} have already greatly outperformed state-of-the-art model-free methods in sample efficiency as shown in their papers. Generally, model-based methods enjoy a significant advantage in sample efficiency over model-free methods. So we can safely expect that our algorithm can also outperform other model-free methods with exploration trick like \citet{tang2017exploration}, \citet{azizzadenesheli2018} and \citet{bellemare2016unifying}.

From the experimental results, it can be verified that our algorithm better captures the model uncertainty, and makes better use of uncertainty using posterior sampling. In our methods, by sampling from a Bayesian linear regression on a fitted feature space, and optimizing under the same sampled MDP in the whole episode instead of re-sampling at every step, the performance of our algorithm is guaranteed from a Bayesian view as analyzed in Section \ref{analysis}. While PETS and MBPO use bootstrapped ensembles of models with a limited ensemble size to "simulate" a Bayesian model, in which the convergence of the uncertainty is not guaranteed and is highly dependent on the training of the neural network. 
However, in our method, there is a limitation of using MPC, which might fail in even higher-dimensional tasks as shown in \citet{janner2019trust}. Incorporating policy gradient techniques for action-selection might further improve the performance and we leave it for future work.

\section{Conclusion}
In our paper, we show that the regret for PSRL algorithm with function approximation can be polynomial in $d, H$ with the assumption that true rewards and transitions (with or without feature embedding) can be modeled by GP with linear kernels. While matching the order of best-known bounds in UCB-based works, PSRL also enjoys computational tractability compared to UCB methods. Moreover, we propose MPC-PSRL in continuous environments, and experiments show that our algorithm exceeds existing model-based and model-free methods with more efficient exploration.

\bibliography{example_paper}

\begin{thebibliography}{32}
\providecommand{\natexlab}[1]{#1}
\providecommand{\url}[1]{\texttt{#1}}
\expandafter\ifx\csname urlstyle\endcsname\relax
  \providecommand{\doi}[1]{doi: #1}\else
  \providecommand{\doi}{doi: \begingroup \urlstyle{rm}\Url}\fi

\bibitem[Ayoub et~al.(2020)Ayoub, Jia, Szepesvari, Wang, and
  Yang]{pmlr-v119-ayoub20a}
Ayoub, A., Jia, Z., Szepesvari, C., Wang, M., and Yang, L.
\newblock Model-based reinforcement learning with value-targeted regression.
\newblock In III, H.~D. and Singh, A. (eds.), \emph{Proceedings of the 37th
  International Conference on Machine Learning}, volume 119 of
  \emph{Proceedings of Machine Learning Research}, pp.\  463--474. PMLR, 13--18
  Jul 2020.
\newblock URL \url{http://proceedings.mlr.press/v119/ayoub20a.html}.

\bibitem[Azar et~al.(2017)Azar, Osband, and Munos]{azar2017minimax}
Azar, M.~G., Osband, I., and Munos, R.
\newblock Minimax regret bounds for reinforcement learning.
\newblock \emph{arXiv preprint arXiv:1703.05449}, 2017.

\bibitem[Azizzadenesheli et~al.(2018)Azizzadenesheli, Brunskill, and
  Anandkumar]{azizzadenesheli2018efficient}
Azizzadenesheli, K., Brunskill, E., and Anandkumar, A.
\newblock Efficient exploration through bayesian deep q-networks.
\newblock In \emph{2018 Information Theory and Applications Workshop (ITA)},
  pp.\  1--9. IEEE, 2018.

\bibitem[Barth-Maron et~al.(2018)Barth-Maron, Hoffman, Budden, Dabney, Horgan,
  Tb, Muldal, Heess, and Lillicrap]{barth2018distributed}
Barth-Maron, G., Hoffman, M.~W., Budden, D., Dabney, W., Horgan, D., Tb, D.,
  Muldal, A., Heess, N., and Lillicrap, T.
\newblock Distributed distributional deterministic policy gradients.
\newblock \emph{arXiv preprint arXiv:1804.08617}, 2018.

\bibitem[Bellemare et~al.(2016)Bellemare, Srinivasan, Ostrovski, Schaul,
  Saxton, and Munos]{bellemare2016unifying}
Bellemare, M.~G., Srinivasan, S., Ostrovski, G., Schaul, T., Saxton, D., and
  Munos, R.
\newblock Unifying count-based exploration and intrinsic motivation.
\newblock \emph{arXiv preprint arXiv:1606.01868}, 2016.

\bibitem[Botev et~al.(2013)Botev, Kroese, Rubinstein, and
  L’Ecuyer]{botev2013cross}
Botev, Z.~I., Kroese, D.~P., Rubinstein, R.~Y., and L’Ecuyer, P.
\newblock The cross-entropy method for optimization.
\newblock In \emph{Handbook of statistics}, volume~31, pp.\  35--59. Elsevier,
  2013.

\bibitem[Camacho \& Alba(2013)Camacho and Alba]{camacho2013model}
Camacho, E.~F. and Alba, C.~B.
\newblock \emph{Model predictive control}.
\newblock Springer Science \& Business Media, 2013.

\bibitem[Chowdhury \& Gopalan(2019)Chowdhury and Gopalan]{chowdhury2019online}
Chowdhury, S.~R. and Gopalan, A.
\newblock Online learning in kernelized markov decision processes.
\newblock In \emph{The 22nd International Conference on Artificial Intelligence
  and Statistics}, pp.\  3197--3205, 2019.

\bibitem[Chua et~al.(2018)Chua, Calandra, McAllister, and Levine]{chua2018deep}
Chua, K., Calandra, R., McAllister, R., and Levine, S.
\newblock Deep reinforcement learning in a handful of trials using
  probabilistic dynamics models.
\newblock In \emph{Advances in Neural Information Processing Systems}, pp.\
  4754--4765, 2018.

\bibitem[Deisenroth \& Rasmussen(2011)Deisenroth and
  Rasmussen]{deisenroth2011pilco}
Deisenroth, M. and Rasmussen, C.~E.
\newblock Pilco: A model-based and data-efficient approach to policy search.
\newblock In \emph{Proceedings of the 28th International Conference on machine
  learning (ICML-11)}, pp.\  465--472, 2011.

\bibitem[Gal et~al.(2016)Gal, McAllister, and Rasmussen]{gal2016improving}
Gal, Y., McAllister, R., and Rasmussen, C.~E.
\newblock Improving pilco with bayesian neural network dynamics models.
\newblock In \emph{Data-Efficient Machine Learning workshop, ICML}, volume~4,
  pp.\ ~34, 2016.

\bibitem[Haarnoja et~al.(2018)Haarnoja, Zhou, Abbeel, and
  Levine]{haarnoja2018soft}
Haarnoja, T., Zhou, A., Abbeel, P., and Levine, S.
\newblock Soft actor-critic: Off-policy maximum entropy deep reinforcement
  learning with a stochastic actor.
\newblock \emph{arXiv preprint arXiv:1801.01290}, 2018.

\bibitem[Jaksch et~al.(2010)Jaksch, Ortner, and Auer]{jaksch2010near}
Jaksch, T., Ortner, R., and Auer, P.
\newblock Near-optimal regret bounds for reinforcement learning.
\newblock \emph{Journal of Machine Learning Research}, 11\penalty0
  (Apr):\penalty0 1563--1600, 2010.

\bibitem[Janner et~al.(2019)Janner, Fu, Zhang, and Levine]{janner2019trust}
Janner, M., Fu, J., Zhang, M., and Levine, S.
\newblock When to trust your model: Model-based policy optimization.
\newblock In \emph{Advances in Neural Information Processing Systems}, pp.\
  12498--12509, 2019.

\bibitem[Jin et~al.(2018)Jin, Allen-Zhu, Bubeck, and Jordan]{jin2018q}
Jin, C., Allen-Zhu, Z., Bubeck, S., and Jordan, M.~I.
\newblock Is q-learning provably efficient?
\newblock In \emph{Advances in Neural Information Processing Systems}, pp.\
  4863--4873, 2018.

\bibitem[Jin et~al.(2020)Jin, Yang, Wang, and Jordan]{jin2020provably}
Jin, C., Yang, Z., Wang, Z., and Jordan, M.~I.
\newblock Provably efficient reinforcement learning with linear function
  approximation.
\newblock In \emph{Conference on Learning Theory}, pp.\  2137--2143, 2020.

\bibitem[Kamthe \& Deisenroth(2018)Kamthe and Deisenroth]{kamthe2018data}
Kamthe, S. and Deisenroth, M.
\newblock Data-efficient reinforcement learning with probabilistic model
  predictive control.
\newblock In \emph{International Conference on Artificial Intelligence and
  Statistics}, pp.\  1701--1710. PMLR, 2018.

\bibitem[Kamyar~Azizzadenesheli(2018)]{azizzadenesheli2018}
Kamyar~Azizzadenesheli, Emma~Brunskill, A.~A.
\newblock Efficient exploration through bayesian deep q-networks, 2018.

\bibitem[Osband \& Van~Roy(2014)Osband and Van~Roy]{osband2014model}
Osband, I. and Van~Roy, B.
\newblock Model-based reinforcement learning and the eluder dimension.
\newblock In \emph{Advances in Neural Information Processing Systems}, pp.\
  1466--1474, 2014.

\bibitem[Osband \& Van~Roy(2017)Osband and Van~Roy]{osband17a}
Osband, I. and Van~Roy, B.
\newblock Why is posterior sampling better than optimism for reinforcement
  learning?
\newblock In Precup, D. and Teh, Y.~W. (eds.), \emph{Proceedings of the 34th
  International Conference on Machine Learning}, pp.\  2701--2710,
  International Convention Centre, Sydney, Australia, 2017. PMLR.

\bibitem[Osband et~al.(2013)Osband, Benjamin, and Daniel]{osband2013}
Osband, I., Benjamin, V.~R., and Daniel, R.
\newblock ({M}ore) efficient reinforcement learning via posterior sampling.
\newblock In \emph{Proceedings of the 26th International Conference on Neural
  Information Processing Systems - Volume 2}, NIPS'13, pp.\  3003--3011, USA,
  2013. Curran Associates Inc.

\bibitem[Osband et~al.(2019)Osband, Van~Roy, Russo, and Wen]{osband2019deep}
Osband, I., Van~Roy, B., Russo, D.~J., and Wen, Z.
\newblock Deep exploration via randomized value functions.
\newblock \emph{Journal of Machine Learning Research}, 20\penalty0
  (124):\penalty0 1--62, 2019.

\bibitem[Rasmussen(2003)]{rasmussen2003gaussian}
Rasmussen, C.~E.
\newblock Gaussian processes in machine learning.
\newblock In \emph{Summer School on Machine Learning}, pp.\  63--71. Springer,
  2003.

\bibitem[Rigollet \& H{\"u}tter(2015)Rigollet and H{\"u}tter]{rigollet2015high}
Rigollet, P. and H{\"u}tter, J.-C.
\newblock High dimensional statistics.
\newblock \emph{Lecture notes for course 18S997}, 2015.

\bibitem[Riquelme et~al.(2018)Riquelme, Tucker, and Snoek]{riquelme2018deep}
Riquelme, C., Tucker, G., and Snoek, J.
\newblock Deep bayesian bandits showdown: An empirical comparison of bayesian
  deep networks for thompson sampling.
\newblock \emph{arXiv preprint arXiv:1802.09127}, 2018.

\bibitem[Russo \& Van~Roy(2014)Russo and Van~Roy]{russo2014learning}
Russo, D. and Van~Roy, B.
\newblock Learning to optimize via posterior sampling.
\newblock \emph{Mathematics of Operations Research}, 39\penalty0 (4):\penalty0
  1221--1243, 2014.

\bibitem[Srinivas et~al.(2012)Srinivas, Krause, Kakade, and
  Seeger]{srinivas2012}
Srinivas, N., Krause, A., Kakade, S.~M., and Seeger, M.~W.
\newblock Information-theoretic regret bounds for gaussian process optimization
  in the bandit setting.
\newblock \emph{IEEE Transactions on Information Theory}, 58\penalty0
  (5):\penalty0 3250--3265, 2012.

\bibitem[Tang et~al.(2017)Tang, Houthooft, Foote, Stooke, Chen, Duan, Schulman,
  De~Turck, and Abbeel]{tang2017exploration}
Tang, H., Houthooft, R., Foote, D., Stooke, A., Chen, X., Duan, Y., Schulman,
  J., De~Turck, F., and Abbeel, P.
\newblock \# exploration: A study of count-based exploration for deep
  reinforcement learning.
\newblock In \emph{31st Conference on Neural Information Processing Systems
  (NIPS)}, volume~30, pp.\  1--18, 2017.

\bibitem[Thompson(1933)]{thompson1933likelihood}
Thompson, W.~R.
\newblock On the likelihood that one unknown probability exceeds another in
  view of the evidence of two samples.
\newblock \emph{Biometrika}, 25\penalty0 (3/4):\penalty0 285--294, 1933.

\bibitem[Williams \& Vivarelli(2000)Williams and Vivarelli]{williams2000upper}
Williams, C.~K. and Vivarelli, F.
\newblock Upper and lower bounds on the learning curve for gaussian processes.
\newblock \emph{Machine Learning}, 40\penalty0 (1):\penalty0 77--102, 2000.

\bibitem[Yang \& Wang(2019)Yang and Wang]{yang2019reinforcement}
Yang, L.~F. and Wang, M.
\newblock Reinforcement learning in feature space: Matrix bandit, kernels, and
  regret bound.
\newblock \emph{arXiv preprint arXiv:1905.10389}, 2019.

\bibitem[Zanette et~al.(2020)Zanette, Lazaric, Kochenderfer, and
  Brunskill]{zanette2020learning}
Zanette, A., Lazaric, A., Kochenderfer, M., and Brunskill, E.
\newblock Learning near optimal policies with low inherent bellman error.
\newblock \emph{arXiv preprint arXiv:2003.00153}, 2020.

\end{thebibliography}
\bibliographystyle{icml2021}
\nocite{jaksch2010near}
\nocite{azar2017minimax}
\nocite{jin2018q}
\nocite{osband2013}  
\nocite{osband2019deep}
\clearpage
\onecolumn
\icmltitle{Supplemental Materials: Model-based Reinforcement Learning for Continuous Control with Posterior Sampling}
\appendix

\section{Proof of Lemma 1}
\label{appendix}
\paragraph{We first prove the results in $\mathbb{R}^d$ when $d=1$:}

Let $p_1(x)$, $p_2(x) $ be the probability density functions for $P_1,P_2$ respectively.  By the property of symmetric distribution, we have that $p_1(\mu_1 + x) = p_1(\mu_1-x)$ and $p_2(\mu_2 + x) = p_2(\mu_2-x)$ for any $x\in \mathbb{R}^d$. Let $\mu_2>\mu_1$ without loss of generality.

Since $\bm{\epsilon_1}$ and $\bm{\epsilon_1}$ share the same distribution, we have that $p_1(x)=p_2(x+\mu_2-\mu_1)$, $\forall x\in\mathbb{R}^d$.


Note that $p_1(x) = p_2(x)$ at $x =\frac{\mu_1 + \mu_2}{2} $. Thus the total variation difference between $p_1$ and $p_2$ can be simplified as twice the integration of one side due to symmetry:
\begin{equation}
\begin{split}
&\int_{-\infty}^{\infty}|p_2(x) - p_1(x)| dx 
=  \int_{-\infty}^{ \frac{\mu_1 + \mu_2}{2}}|p_2(x) - p_1(x)|dx +\int_{ \frac{\mu_1 + \mu_2}{2}}^{\infty} |p_2(x) - p_1(x)| dx =2\int_{ \frac{\mu_1 + \mu_2}{2}}^{\infty}|p_2(x) - p_1(x)|dx,
\end{split}
\end{equation}
where the last equation come from 
\begin{equation}
\begin{split}
&p_1(\frac{\mu_1 + \mu_2}{2}-x)=p_1(\mu_1-x+\frac{\mu_2-\mu_1}{2})=p_1(\mu_1+x-\frac{\mu_2-\mu_1}{2})\\
&=p_2(\mu_2+x-\frac{\mu_2-\mu_1}{2})=p_2(\frac{\mu_1 + \mu_2}{2}+x).
\end{split}
\end{equation}

Case 1: If the density functions $p_1$ and $p_2$ are unimodal, we have $p_2(x)>p_1(x)$ when $x>\frac{\mu_1 + \mu_2}{2}$.
Let $z_1 = x - \mu_1, z_2 = x - \mu_2$, we have:
\begin{equation}
\begin{split}
& \int_{ \frac{\mu_1 + \mu_2}{2}}^{\infty} |p_2(x) - p_1(x)| dx \\
&= \int_{ \frac{\mu_1 - \mu_2}{2}}^{\infty}p_2(z_2)dz_2 -  \int_{ \frac{\mu_2 - \mu_1}{2}}^{\infty} p_1(z_1)dz_1\\
&=\int_{ \frac{\mu_1 - \mu_2}{2}}^{\frac{\mu_2 - \mu_1}{2}} p_2(z_2)dz_2\\
& \leq  \int_{ \frac{\mu_1 - \mu_2}{2}}^{\frac{\mu_2 - \mu_1}{2}} p_{max} dz= p_{max}|\mu_2 - \mu_1|,\\ 
\end{split}
\end{equation}

where $p_{max}$ is the maximum probiblity density of $p_2$, which is dependent on the variance of the shared noise distribution. The proof is completed by combing (12) and (14).

Case 2: When $p_1(x),p_2(x)$ are not unimodal, there exist $C_0$ such that $p_2(x)$ would be a descreasing function in $x$ when $x>\frac{\mu_1 + \mu_2}{2}+C_0(\mu_2-\mu_1)$ (otherwise the integration of the density cannot be 1, and $C_0$ is a constant which is dependent on the specific distribution). 
Recall that $p_1(x)=p_2(x+\mu_2-\mu_1)$, so when $x>\frac{\mu_1 + \mu_2}{2}+C_0(\mu_2-\mu_1)$,  $p_2(x)>p_2(x+\mu_2-\mu_1)=p_1(x)$. Let $z_1 = x - \mu_1, z_2 = x - \mu_2$, we have

\begin{equation}
\begin{split}
  & \int_{ \frac{\mu_1 + \mu_2}{2}}^{\infty} |p_2(x) - p_1(x)| dx \\
  &=\int_{ \frac{\mu_1 + \mu_2}{2}}^{\frac{\mu_1 + \mu_2}{2}+C_0(\mu_2-\mu_1)} |p_2(x) - p_1(x)| dx +\int_{ \frac{\mu_1 + \mu_2}{2}+C_0(\mu_2-\mu_1)}^{\infty} |p_2(x) - p_1(x)|dx\\
  &\leq C_0p_{max}(\mu_2-\mu_1) + \int_{ \frac{\mu_1 - \mu_2}{2}+C_0(\mu_2-\mu_1)}^{\infty}p_2(z_2)dz_2 -  \int_{ \frac{\mu_2 - \mu_1}{2}+C_0(\mu_2-\mu_1)}^{\infty} p_1(z_1)dz_1\\\
  & \leq C_0p_{max}(\mu_2-\mu_1) + \int_{ \frac{\mu_1 - \mu_2}{2}+C_0(\mu_2-\mu_1)}^{\frac{\mu_2 - \mu_1}{2}+C_0(\mu_2-\mu_1)} p_{max} dz= (C_0+1)p_{max}|\mu_2 - \mu_1|,\\ 
\end{split}
\end{equation}
then the proof is complete by combining (12) and (15).
\paragraph{Now we extend the result to 
$\mathbb{R}^d (d \geq 2)$:}

Let the shared covariance matrix for the overall noise distribution be $\sigma^2\bm{I}_d$, where the noise in each dimension is drawn independently with variance $\sigma^2$.
We can rotate the coordinate system recursively to align the last axis with vector $ \bm{\mu_1} -  \bm{\mu_2}$, such that the coordinates of $ \bm{\mu_1}$ and $ \bm{\mu_2}$ can be written as $(0, 0, \cdots, 0, \hat{\mu}_1)$, and $(0, 0, \cdots,0, \hat{\mu}_2)$ respectively, with $|\hat{\mu}_2-\hat{\mu}_1| = \left\lVert \bm{ \mu_2-  \mu_1} \right\rVert_2$. 

The new covariance matrix after rotation will still be $\sigma^2 \bm{I}_d$ since the rotation matrix is orthogonal. Notice that rotation is a linear transformation on the original noises, and the original noises are independently drawn from each axis, so the new covariance matrix indicates the noises in each new axis (after rotation) can also be viewed as independent\footnote{If noises in each new axis are not independent, they can only be linearly related, which would result in non-zero covariance and causes contradiction.}. Without loss of generality, let $\bm{\hat{\mu}_1} \geq\bm{ \hat{\mu}_2}$. Using $p'_1,p'_2$ to indicate the marginal probability density in d-th dimension, we have:

\begin{equation}
\begin{split}
&\int_{-\infty}^{\infty}\int_{-\infty}^{\infty} \cdots \int_{-\infty}^{\infty} |p_2( \bm{x}) - p_1( \bm{x})| dx_1dx_2 \cdots dx_d\\
&=\int_{-\infty}^{\infty}|p'_2(x_d)-p'_1(x_d)|dx_d\\
\end{split}
\end{equation}
Then we can follow the same steps in $\mathbb{R}^1$ to finish the proof.

\paragraph{Remark} For Gaussian noises with shared covariance $\sigma^2\bm{I}_d$,  Pinsker's inequality and the KL-divergence of two Gaussian distributions can also show that $\int |p_1( \bm{x})-p_2( \bm{x})|d \bm{x} \leq\frac{1}{\sigma} ||\bm{\mu}_1- \bm{\mu}_2||_2$. But our upper bound for Gaussian noises is tighter: we have $\int |p_1( \bm{x})-p_2( \bm{x})|d \bm{x} \leq \sqrt{\frac{2}{\pi \sigma^2}}|| \bm{\mu}_1- \bm{\mu}_2||_2$. Also here we develop upper bounds for a wide class of symmetric distributions.



\section{Detailed comparison with previous works in Section 3.4}
Here we compare our result with Corollary 2 in \cite{osband2014model}.
In their Corollary 2 of linear quadratic systems, the regret bound is $\tilde{O}(\sigma C\lambda_1n^2\sqrt{T})$, where $\lambda_1$ is the largest eigenvalue of the matrix $Q$ in the optimal value function $V_1(s) = s^TQs$, where $V_1$ denotes the value function counting from step 1 to H within an episode, $s$ is the initial state, reward at the $i$-th step $r_{i} = s_i^TPs_i+a_i^TRa_i + \epsilon_{P,i}$, and the state at the  $i+1$-th step $s_{i+1} =  As_i+Ba_i + \epsilon_{P,i}$ , $i\in [H]$. However, the largest eigenvalue of $Q$ is actually exponential in $H$: Recall the Bellman equation we have $V_i(s_i) = \min_{a_i} \mathbb{E}[s_i^TPs_i + a_i^TRa_i + \epsilon_{P,i} +V_{i+1}(As_i+Ba_i+\epsilon_{P,i})]$, $V_{H+1}(s)=0$. Thus in $V_1(s)$, we can observe a term of $(A^{H-1}s)^TP(A^{H-1}s)$, and the eigenvalue of the matrix $(A^{H-1})^TPA^{H-1}$ is exponential in $H$. 

Even if we change the reward function from quadratic to linear, say $r_{i} = s_i^TP+a_i^TR + \epsilon_{P,i}$ the Lipschitz constant of the optimal value function is still exponential in $H$ since there is still a term of $(A^{H-1}s)^TP$ in $V_1(s)$. \citet{chowdhury2019online} maintain the assumption of this Lipschitz property, thus there exists $\mathbb{E}[L^*]$ in their bound. As a result, there is still no clear dependency on $H$ in their regret, and in their Corollary 2 of LQR, they follow the same steps as \citet{osband2014model}, and still maintain a term with $\lambda_1$, which is actually exponential in $H$ as discussed. Although \citet{osband2014model} mention that system noise helps to smooth future values, but they do not explore it although the noise is assumed to be subgaussian. The authors directly use the Lipschitz continuity of the underlying function in the analysis of LQR, thus they have an exponential bound on $H$ which is very loose, and it can be improved by our analysis. \cite{chowdhury2019online} do not explore how the system noise can improve the theoretical bound either. 

\section{Details for bounding the sum of posterior variances in Section 3.4}

Here we slightly modify the Proof of Lemma 5.4 in \cite{srinivas2012} and show that $\Sigma_{i=1}^{n} \sigma_i^{2}(h_{i}) = \mathcal{O} ((d_s+d_a)\log(n))$, then we can write $\Sigma_{k=1}^{[\frac{T}{H}]} \sigma_k^{'2}(h_{\text{kmax}}) = \mathcal{O} ((d_s+d_a)\log[\frac{T}{H}])$ with just changes of notations.

For any $s^2\in[0,\sigma_f^{-2} C_1]$ we have $s^2\leq C_2\log(1+s^2)$, where $C_2 = \frac{\sigma_f^{-2}C_1}{\log(1+\sigma_f^{-2}C_1)}$. We treat $C_1$ as the upper bound of the variance (note that the bounded variance property for linear kernels only requires the range of all state-action pairs actually encountered in $M^*$ not to expand to infinity as T grows, which holds in general episodic MDPs).

Lemma 5.3 in \cite{srinivas2012} shows that the information gain for dataset $\{h_1,...,h_n\}$ is equal to $\frac{1}{2}\Sigma_{i=1}^n log(1+\sigma_f^{-2}\sigma_i^{2}(h_{i}))$, and we also have $\sigma_i^{2}(h_{i})\leq C_2\log(1+\sigma_i^{2}(h_{i}))$. Thus we can use the upper bound of the information gain of linear kernels, which is $\mathcal{O} ((d_s+d_a)\log(n))$ as presented in Theorem 5 in \cite{srinivas2012}, to upper bound the sum of posterior variances.

\section{Handling $\mathbb{E}[\Sigma_{k=1}^{[\frac{T}{H}]}\tilde{\Delta}_k(r)|\mathcal{H}_{k}]$ in Section 3.4}

Recall that $\tilde{\Delta}_k(r) = \Sigma_{i=1}^H (\bar{r}^k(h_i)-\bar{r}^*(h_i))$, so we can omit section 3.2 and directly follow steps in section 3.3 to derive another upper bound which is similar to Lemma 3: $[\Sigma_{k=1}^{[\frac{T}{H}]}\tilde{\Delta}_k(r)|\mathcal{H}_{k}]\leq \Sigma_{k=1}^{[\frac{T}{H}]}2\sqrt{\sigma_k(h_{\text{kmax}})\log(\frac{4T}{\delta})}$  with probability at least $1-\delta$. 

Then we can follow similar steps in section 3.4 to develop that  $\mathbb{E}[\Sigma_{k=1}^{[\frac{T}{H}]}\tilde{\Delta}_k(r)|\mathcal{H}_{k}]= \tilde{\mathcal{O}}(\sqrt{dHT})$.



\section{Experimental details}

For model-based baselines, the average number of episodes required for convergence is presented below (convergence results for Cartpole and Pendulum can be found in Figure 1 and Figure 2 in the main paper):
\begin{table}[h]
    \centering
    \vspace{-0.3cm}

\begin{tabular}{l|cccc}
\toprule
Method & Reacher(r) & Pusher(r) &  Reacher & Pusher\\  \midrule
 Ours &\textbf{20.2}&\textbf{146.6}&\textbf{29.8}&\textbf{151.6} \\
 MBPO &34.6&209.4&54.2& 225.0 \\
PETS  &26.2&193.4& - & - \\
\bottomrule
\end{tabular}

    \vspace{-0.3cm}
\end{table}
For model-free methods, we have provided the convergence results in Figure \ref{fig} for SAC (blue dots). Model-free methods generally converge after 100 episodes for Cartpole and Pendulum, and around 1000 episodes for Pusher and Reacher.

Hyperparameters for MBPO:
\begin{table}[H]
\centering
\begin{tabular}{|l|l|l|l|l|}
\hline
env                                                                   & cartpole & pendulum & pusher & reacher \\ \hline
\begin{tabular}[c]{@{}l@{}}env steps \\ per episode\end{tabular}      & 200      & 200      & 150    & 150     \\ \hline
\begin{tabular}[c]{@{}l@{}}model rollouts\\ per env step\end{tabular} & \multicolumn{4}{l|}{400}               \\ \hline
ensemble size                                                         & \multicolumn{4}{l|}{5}                 \\ \hline
\begin{tabular}[c]{@{}l@{}}network \\ architecture\end{tabular} &
  \begin{tabular}[c]{@{}l@{}}MLP with \\ 2 hidden layers \\ of size 200\end{tabular} &
  \begin{tabular}[c]{@{}l@{}}MLP with\\ 2 hidden layers\\ of size 200\end{tabular} &
  \begin{tabular}[c]{@{}l@{}}MLP with\\ 4 hidden layers \\ of size 200\end{tabular} &
  \begin{tabular}[c]{@{}l@{}}MLP with\\ 4 hidden layers\\ of size 200\end{tabular} \\ \hline
\begin{tabular}[c]{@{}l@{}}policy updates\\ per env step\end{tabular} & \multicolumn{4}{l|}{40}                \\ \hline
model horizon &
  \begin{tabular}[c]{@{}l@{}}1-\textgreater{}15 from \\ episode 1-\textgreater{}30\end{tabular} &
  \begin{tabular}[c]{@{}l@{}}1-\textgreater{}15 from \\ episode 1-\textgreater{}30\end{tabular} &
  1 &
  \begin{tabular}[c]{@{}l@{}}1-\textgreater{}15 from \\ episode 1-\textgreater{}30\end{tabular} \\ \hline
\end{tabular}
\caption{Hyperparamters for MBPO}
\label{tab:my-table}
\end{table}
And we provide hyperparamters for MPC and neural networks in PETS:
\begin{table}[H]
\centering
\begin{tabular}{|l|l|l|}
\hline
env                                                             & pusher    & reacher    \\ \hline
\begin{tabular}[c]{@{}l@{}}env steps\\ per episode\end{tabular} & 150       & 150        \\ \hline
popsize                                                         & 500       & 400        \\ \hline
\begin{tabular}[c]{@{}l@{}}number\\ of elites\end{tabular}      & 50        & 40         \\ \hline
\begin{tabular}[c]{@{}l@{}}network \\ architecture\end{tabular} & \multicolumn{2}{l|}{\begin{tabular}[c]{@{}l@{}}MLP with\\ 4 hidden layers\\ of size 200\end{tabular}} \\ \hline
\begin{tabular}[c]{@{}l@{}}planning\\ horizon\end{tabular}      & 25        & 25         \\ \hline
max iter                                                        & \multicolumn{2}{l|}{5} \\ \hline
ensemble size                                                   & \multicolumn{2}{l|}{5} \\ \hline
\end{tabular}
\caption{Hyperparamters for PETS}
\end{table}
Below are hyperparameters of our planning algorithm, which is the same with PETS, except for ensemble size (since we do not need ensembled models, hence our ensemble size is actually 1):

\begin{table}[H]
\centering
\begin{tabular}{|l|l|l|l|l|}
\hline
env                                                             & cartpole & pendulum & pusher & reacher \\ \hline
\begin{tabular}[c]{@{}l@{}}env steps\\ per episode\end{tabular} & 200      & 200      & 150    & 150     \\ \hline
popsize                                                         & 500      & 100      & 500    & 400     \\ \hline
\begin{tabular}[c]{@{}l@{}}number\\ of elites\end{tabular}      & 50       & 5        & 50     & 40      \\ \hline
\begin{tabular}[c]{@{}l@{}}network \\ architecture\end{tabular} &
  \begin{tabular}[c]{@{}l@{}}MLP with \\ 2 hidden layers\\ of size 200\end{tabular} &
  \begin{tabular}[c]{@{}l@{}}MLP with\\ 2 hidden layers\\ of size 200\end{tabular} &
  \begin{tabular}[c]{@{}l@{}}MLP with\\ 4 hidden layers\\ of size 200\end{tabular} &
  \begin{tabular}[c]{@{}l@{}}MLP with\\ 4 hidden layers\\ of size 200\end{tabular} \\ \hline
\begin{tabular}[c]{@{}l@{}}planning\\ horizon\end{tabular}      & 30       & 20       & 25     & 25      \\ \hline
max iter                                                        & \multicolumn{4}{l|}{5}                 \\ \hline
\end{tabular}
\caption{Hyperparamters for our method}
\end{table}


For SAC and DDPG, we use the open-source code ( \url{https://github.com/dongminlee94/deep_rl}) for implementation without changing their hyperparameters. Here we thank the authors for sharing the code!

We run all the experiments on a single NVIDIA GeForce RTX-2080Ti GPU. For smaller environments like Cartpole and Pendulum, all experiments are done within an hour to run 150k steps. For Reacher and Pusher, our algorithm takes about four hours to run 150k steps, while PETS and MBPO take about three hours to run 150k steps (our extra computation cost comes from Bayesian update, and we plan to explore acceleration for that as future work). SAC and DDPG take about one hour for training 150k steps which is much faster than other baselines since they are model-free algorithms and no need to train models.


\end{document}